\documentclass[11pt]{article} 

\usepackage{times}
\usepackage{fancyvrb}
\usepackage{relsize}
\usepackage{graphicx}
\usepackage{amsmath}
\usepackage{hyperref}

\newtheorem{lemma}{Lemma}
\newtheorem{theorem}{Theorem}
\newtheorem{proof}{Proof}
\newtheorem{corollary}{Corollary}

\begin{document}

\title{Improved Estimation of Class Prior Probabilities through Unlabeled
Data}

\author{Norman Matloff \\
       Department of Computer Science \\
       University of California, Davis \\
       Davis, CA  95616}


\maketitle

\begin{abstract} 

\noindent Work in the classification literature has shown that in
computing a classification function, one need not know the class
membership of all observations in the training set; the unlabeled
observations still provide information on the marginal distribution of
the feature set, and can thus contribute to increased classification
accuracy for future observations.  The present paper will show that this
scheme can also be used for the estimation of class prior probabilities,
which would be very useful in applications in which it is difficult or
expensive to determine class membership.  Both parametric and
nonparametric estimators are developed.  Asymptotic distributions of the
estimators are derived, and it is proven that the use of the unlabeled
observations does reduce asymptotic variance.  This methodology is also
extended to the estimation of subclass probabilities.  

\end{abstract}

%
%

\section{Introduction}
\label{intro}

There has been much work on the issue of unlabeled data in
classification problems.  Some papers, such as \cite{liu}, have taken
the point of view that the data is missing, due to some deficiency in
the data collection process, while others, such as \cite{nigam},
are aimed at situations in which some observations are
deliberately left unlabeled. 

The latter approach is motivated by the fact that in many applications
it is very difficult or expensive to determine class membership.  Thus
\cite{nigam} proposed that in part of the training set, class membership
be left undetermined.  The unlabeled observations would still provide
information on the marginal distribution of the features, and it was
shown that this information can contribute to increased classification
accuracy for future observations.

In the present work, it is again assumed that we have both labeled and
unlabeled data, but the focus is on estimation of the class prior
probabilities rather than estimation of the classification function.  We
wish to estimate those probabilities via a mixture of labeled and
unlabeled data, in order to economize on the time and effort needed to
acquire labeling information.

For example, consider the geographic application in \cite{blackard}.
Our class variable here is forest cover type, representing one of seven
classes.  Suppose we wish to estimate the population proportions of the
various cover types.  The authors note that a problem arises in that
cover data is generally ``directly recorded by field personnel or
estimated from remotely sensed data...[both of which] may be
prohibitively time consuming and/or costly in some situations.''
However, various feature variables are easily recorded, such as
elevation, horizontal distance to the nearest roadway and so on.
Previous work in the estimation of classification functions suggests
that we may use this other data, without labeling, as a means of
reducing the time and effort needed to determine class membership.

As another example, consider the patent analysis reported in
\cite{narin}.  The overall goal was to estimate the proportion of
patents whose research was publicly funded.  So, we would have just two
classes, indicating public funding or lack of it.  The determination of
which patents were publicly funded involved inspection of not only the
patents themselves, but also the papers cited in the patents, a very
time-consuming, human-intensive task.  However, using the approach
discussed here, we could define our features to consist
of some key words which have some predictive power for public funding
status, such as appearance of the word {\it university} in the Assignee
section of the patent document, and then use the feature data to help
achieve our goal of estimating the class probabilities.

Specifically, one might develop a classification rule from the labeled
observations, and then use the rule to classify the remaining
observations in the training set.  In other words, we would find
predicted forest cover types for each of the unlabeled observations.
Finally, one would obtain estimates for the proportions of forest cover
types in the full training set, using the predicted cover types for
those on which labels had not been collected.  Actually, one can do even
better by using estimated conditional class probabilities of the
unlabeled and labeled observations, as will be proven here.

Let us set some notation.  Suppose there are c classes, and let $Y =
(Y^{(1)},...,Y^{(c)})^T$ be the class identification vector, so that
$Y^{(j)}$ is 1 or 0, according to whether the item belongs to class
j.\footnote{This notation, traditional in the statistics literature,
differs from the machine learning custom of having $Y^{(j)}$ take on the
values $\pm 1$.  However, the statistics notation will be convenient
here in various ways, e.g.  that $P(Y^{(j)} = 1)$ is simply $EY^{(j)}$.}
Let $X = (X^{(1)},...,X^{(f)})^T$ be the feature vector.  For some
observations in our training set, we will have data on both $X$ and $Y$,
but others will be unlabeled, i.e. we will have data only on $X$.  In
our context here, our goal is to estimate $q_j = EY^{(j)}$, the prior
probability of class j.

The key point is that the unlabeled observations provide additional
information on the marginal distribution of $X$.  Since in the Law of
Total Expectation, $EY = E[E(Y|X)]$, the outer expectation is with
respect to X, the better our estimate of the marginal distribution of
$X$, the better our estimate of $EY$.  Thus using the two data sets in
concert can give us more accurate estimates of the target quantity,
$EY$, than can the labeled set alone.  In this paper such estimates will
be developed and analyzed.

It is important to note that that not only do we want to be able to
estimate $EY$, but we also may need a measure of the accuracy of our
estimates.  Most journals in the sciences, for example, want statistical
inference to accompany findings, in the form of confidence intervals
and hypothesis tests.  This issue is also addressed in this paper.

In addition, we may be interested in estimating subclass probabilities,
again using labeled and unlabeled data.  For instance, we may wish to
compare the rates of public funding of patents in the computer
field on one hand, and in the optics field on the other.

Section \ref{me} first notes that the statistical literature, such as
\cite{matloff} and \cite{muller}, contains some semiparametric
methodology which is related to our goals.  A typical case here might be
the popular logistic model,\cite{agresti}, especially in settings with
continuous features such as in the forest example.  Asymptotic
distributions are derived for our estimators of class prior
probabilities in Section \ref{mv}, and that section also proves that the
use of unlabeled data does reduce asymptotic variance in estimating the
priors.

The question then arises as to how much reduction in asymptotic variance
can be attained.  An analyst faced with a real data set must decide
whether the method proposed here will work well in his/her particular
setting. Is it worth collecting some unlabeled data?  In Section
\ref{assessing} techniques are developed for assessing this.  These
techniques may also be useful for the original application of unlabeled
data, which was to enhance the accuracy of classification function
estimators.  Section \ref{empirical} assesses how well the method
introduced here for estimating $EY$ works on some real data sets.

The remainder of the paper covers variations.  Section \ref{subclass}
addresses the subclass problem, and Section \ref{discrete} treats the
case of purely discrete features.  Asymptotic inference methodology is
again developed, and again it is proven that asymptotic variance is
reduced.  Finally, Section \ref{dsc} will discuss possible extensions of
these approaches.

\section{Basic Framework for Semiparametric Models}
\label{me}

Denote ithe $i^{th}$ observation in our data by $(Y_i,X_i)$.  Keep in
mind that $Y_i$ and $X_i$ are vectors, of length c and f.  We will
assume that the first r observations are labeled but the remainder are
not.  In other words, $Y_i$ is known for i = 1,...,r but unknown for i =
r+1,...,n.  Let $Y_i^{(j)}$ denote the value of $Y^{(j)}$ in observation
i, i.e. the value of component j in $Y_i$.

As statistical inference is involved, we will be determining asymptotic
distributions of various estimators, as $n \rightarrow \infty$.  It is
assumed that the ratio $r/n$ goes to the limit $\gamma$ as $n
\rightarrow \infty$.

The main analysis will be semiparametric.  This means that we assume
that it has been determined that the conditional class membership
probability functions belong to some parametric family $g(t,\theta)$,
i.e. that

\begin{equation}
P(Y^{(j)} = 1 | X = t) = g(x,\theta_j )
\end{equation}

\noindent
for some vector-valued parameter $\theta_j$, but we assume no
parametric model for the distribution of X.

In the material that follows, the reader may wish to keep in mind the
familiar logistic model, 

\begin{equation}
P(Y^{(j)} = 1 | X = t) = \frac{1}{1+e^{-\theta_j^T t}}
\end{equation}

\noindent
where t includes a constant component and T denotes matrix transpose.

The estimator $\hat{\theta}_j$ for the population value $\theta_j$ is
obtained from $(Y_1,X_1),...,(Y_r,X_r)$ via the usual iteratively
reweighted least squares method, with weight function

\begin{equation}
\label{wdef}
w(t,\theta_j) =
\frac{1}{Var(Y^{(j)}|X=t)} =
\frac{1}
{ g(t,\theta_j)
[1-g(t,\theta_j)]}
\end{equation}

\noindent
That is, the estimator is the root of

\begin{equation}
\label{root}
0 = \sum_{i=1}^r w(X_i,\hat{\theta}_j)
\{
Y_i^{(j)} - g(X_i,\hat{\theta}_j)
\} ~
g'(X_i,\hat{\theta}_j)
\end{equation}

\noindent
where $g'$ is the vector of partial derivatives of g with respect
to its second argument.

\noindent
Routines to perform the computation are commonly available, such as
the {\tt glm} function in the R statistical package, with the argument
{\bf family=binomial}.

Since the Law of Total Expectation implies that 

\begin{equation}
\label{ygivenx}
q_j = EY^{(j)} = 
E \left [ E(Y^{(j)}|X) \right ] =
E[g(X,\theta_j)]
\end{equation}

\noindent
we take as our new estimator of ${q_j}$ 

\begin{equation}
\label{qhatj}
\hat{q}_j = \frac{1}{n} \sum_{i=1}^n g(X_i,\hat{\theta}_j)
\end{equation}

\noindent
This will be compared to the classical estimator based on the labeled
data,

\begin{equation}
\label{ybarj}
\bar{Y}^{(j)} = \frac{1}{r} \sum_{i=1}^r Y_i^{(j)}
\end{equation}

\section{Asymptotic Distribution and Statistical Inference}
\label{mv}

We will assume that g satisfies the necessary smoothness conditions,
such as the ones given in \cite{jennrich}, and that the means exist and
so on. 

\begin{lemma}
\label{thm1}

Define

\begin{equation}
\label{ajdef}
A_j =  E 
\{
w(X,\theta_j) g'(X,\theta_j) g'(X,\theta_j)^T
\}
\end{equation}

\noindent
and

\begin{equation}
B_j = E[g'(X,\theta_j)] 
\end{equation}

\noindent
Then $n^{\frac{1}{2}} (\hat{q}_j - q_j)$ is asymptotically equivalent to

\begin{equation}
\label{qasy}
n^{-\frac{1}{2}}
[
\sum_{i=1}^n
\{ g(X_i,{\theta_j}) - q_j \}  ~ + \\
\gamma^{-\frac{1}{2}}
B_j^T 
A_j^{-1} 
\sum_{i=1}^r 
w(X_i,{\theta_j})
\{
Y_i^{(j)} - g(X_i,{\theta_j})
\} ~
g'(X_i,{\theta_j})
] 
\end{equation}

\end{lemma}

\begin{proof}

\noindent
Write

\begin{equation}
n^{\frac{1}{2}} (\hat{q}_j - q_j) 
=  n^{-\frac{1}{2}} 
\sum_{i=1}^n
\left [ g(X_i,\hat{\theta}_j) - q_j \right ]
\end{equation}

\noindent
Form the Taylor expansion of $\sum_{i=1}^n g(X_i,\hat{\theta}_j)$ 
around the point $\theta_j$. Then there is a $\tilde{\theta}_j$ 
between $\theta_j$ and $\hat{\theta}_j$ such that

\begin{eqnarray}
\label{taylor}
n^{\frac{1}{2}} (\hat{q}_j - q_j) &=&
n^{-\frac{1}{2}}
\sum_{i=1}^n
\left [ g(X_i,{\theta_j}) - q_j \right ]  \\
&+&
\left [  
n^{-1} \sum_{i=1}^n g'(X_i,\tilde{\theta_j}) 
\right ]^T ~
n^{\frac{1}{2}} (\hat{\theta}_j - \theta_j ) \\
&=& C_n + D_n
\end{eqnarray}

\noindent
The $C_n$ term converges in distribution, by the Central Limit Theorem.
What about $D_n$?  It is well established, for example in \cite{jennrich}, 
that $n^{\frac{1}{2}} (\hat{\theta}_j - \theta_j)$ is asymptotically 
equivalent to

\begin{equation}
\label{thetahatequiv}
n^{-\frac{1}{2}} 
\gamma^{-\frac{1}{2}}
A_j^{-1} 
\sum_{i=1}^r w(X_i,{\theta_j})
\{
Y_i^{(j)} - g(X_i,{\theta_j})
\} ~
g'(X_i,{\theta_j})
\end{equation}

\noindent
Meanwhile, 

\begin{equation}
\lim_{n \rightarrow \infty}
\frac{1}{n}
\sum_{i=1}^n g'(X_i,\theta_j) =
B_j \textrm{ w.p. 1} 
\end{equation}

\noindent
Thus $D_n$ is asymptotically equivalent to

\begin{equation}
F_n =
n^{-\frac{1}{2}} 
\gamma^{-\frac{1}{2}}
B_j^T 
A_j^{-1} 
\sum_{i=1}^r 
w(X_i,{\theta_j})
\{
Y_i^{(j)} - g(X_i,{\theta_j})
\} ~
g'(X_i,{\theta_j})
] 
\end{equation}

\noindent
by Slutsky's Theorem [\cite{wasserman}].

So $D_n - F_n$ is $o_p(1)$, and thus the same is true for
$(C_n+D_n)-(C_n+F_n)$.  The result follows.

\end{proof}

We can now determine the asymptotic variance of $\hat{q}_j$:

\begin{theorem}

The asymptotic variance of $\widehat{q}_j$ is

\begin{equation}
\label{av}
AVar(\hat{q}_j) =
\frac{1}{n} 
Var[g(X,\theta_j)] +  \frac{1}{r} B_j^T A_j^{-1} B_j  
\end{equation}

\end{theorem}

\begin{proof}

\noindent
Write (\ref{qasy}) as $G_n + H_r$, where

\begin{equation}
G_n = 
n^{-\frac{1}{2}}
\sum_{i=1}^n
\{ g(X_i,{\theta_j}) - q_j \}  
\end{equation}

\noindent
and

\begin{equation}
H_r = 
r^{-\frac{1}{2}}
B_j^T 
A_j^{-1} 
\sum_{i=1}^r 
w(X_i,{\theta_j})
\{
Y_i^{(j)} - g(X_i,{\theta_j})
\} ~
g'(X_i,{\theta_j})
\end{equation}

Let us first establish that $G_n$ and $H_r$ are uncorrelated.  To show
this, let us find the expected value of the product of two minor terms.
Again applying the Law of Total Expectation, write

\begin{eqnarray}
\label{xprod0}
& & E \left [ g(X,\theta_j) \cdot
w(X,\theta_j)
\{ Y^{(j)} - g(X,\theta_j) \} 
g'(X,\theta_j) \right ] \nonumber \\ 
&=& E \left [ 
E \left (
g(X,\theta_j) 
w(X,\theta_j)
\{ Y(j) - g(X,\theta_j) \} 
g'(X,\theta_j)
| X \right ) \right ] \nonumber \\ 
&=& E 
\left [ 
g(X,\theta_j) 
w(X,\theta_j)
E \left (
Y^{(j)} - g(X,\theta_j) | X \right )
g'(X,\theta_j)
\right ] \nonumber \\
&=& 0
\end{eqnarray}

\noindent
from (\ref{ygivenx}).

For a random vector W, let Cov(W) denote the covariance matrix of W.
The Law of Total Expectation can be used to derive the relation

\begin{equation}
Cov(R) = E[Cov(R|S)] + Cov[E(R|S)]
\end{equation}

\noindent
for random vectors $R$ and $S$.  Apply this with $S = X$ and 

\begin{equation}
R = 
r^{-\frac{1}{2}}
B_j^T 
A_j^{-1} 
w(X,\theta_j)
\{ Y^{(j)} - g(X,\theta_j) \} 
g'(X,\theta_j)
\end{equation}

\noindent
Then $E(R|S) = 0)$.  Also

\begin{eqnarray}
Cov(R|S) &=& 
r^{-1}
B_j^T 
A_j^{-1} 
w(X,\theta_j)^2
\left [
Var(Y|X)
\right ]
g'(X,\theta_j)
g'(X,\theta_j)^T
A_j^{-1} 
B_j \\
&=& 
r^{-1}
B_j^T 
A_j^{-1} 
w(X,\theta_j)
g'(X,\theta_j)
g'(X,\theta_j)^T
A_j^{-1} 
B_j \\
\end{eqnarray}

\noindent
from (\ref{wdef}).  This yields 

\begin{equation}
E[Cov(R|S)]  =
r^{-1}
B_j^T 
A_j^{-1} 
B_j 
\end{equation}

\noindent
from (\ref{ajdef}).

Therefore the asympotic variance of $\hat{q}_j$ is as given in (\ref{av}).

\end{proof}

Statistical inference is then enabled, with $\hat{q}$ being
approximately normally distributed.  The standard error is obtained as
the square root of the estimated version of the quantities in
(\ref{av}), taking for instance

\begin{equation}
\label{varghat}
\widehat{Var}[g(X,\theta_j)] =
\frac{1}{n} \sum_{i=1}^n [g(X_i,\hat{\theta_j}) - \hat{q_j}]^2
\end{equation}

\noindent
An approximate 100(1-$\alpha$)\% confidence interval for $q_j$ is then
obtained by adding and subtracting $\Phi^{-1}(\alpha/2)$ times the
standard error, where $\Phi$ is the standard normal distribution
function.


\begin{corollary}

The asympotic variance of the estimator based on both the labeled and
unlabeled data, $\widehat{q}_j$, is less than or equal to that of the 
estimator based only on labeled data, $\bar{Y}^{(j)}$.  The 
inequality is strict as long as the random variable $g(X,\theta_j)$ 
is not constant.  

\end{corollary}

\begin{proof}

\noindent
If we were to use only the labeled data, the asympotic variance of our
estimator of $q_j$ would be the result of setting n = r in (\ref{av}),
i.e.

\begin{equation}
\label{av2}
AVar(\hat{q}_j) = \frac{1}{r} 
\left [ Var\{g(X,\theta_j)\} + B_j^T A_j^{-1} B_j \right ]
\end{equation}

Comparing (\ref{av2}) and (\ref{av}), we see that use of the unlabeled
data does indeed yield a reduction in asympotic variance, by the amount
of 

\begin{equation}
\label{improvement}
\left ( \frac{1}{r} - \frac{1}{n} \right )  Var\{g(X,\theta_j)\} 
\end{equation}

It was proven in \cite{matloff} that even with no unlabeled data,
the asympotic variance of (\ref{qhatj}) is less than or equal to that of
$\bar{Y}^{(j)}$.  (This result was later sharpened and extended by
\cite{muller}.) The claimed relation then follows.

\end{proof}

\section{Assessing Potential Improvement}
\label{assessing}

Clearly, the stronger the relationship between $X$ and Y, the greater the
potential improvement that can accrue from use of unlabeled data.  This
raises the question of how to assess that potential improvement.

One approach would be to use a measure of the strength of the
relationship between $X$ and $Y$ itself.  A number of such measures have
been proposed, such as those described in \cite{nagelkerke}.  Many aim
to act as an analog of the $R^2$ value traditionally used in linear
regression analysis.  Here we suggest what appears to be a new measure,
whose form more directly evokes the spirit of $R^2$.

Consider first random variables U and V, with V being continuous.  If we
are predicting V from U, the population version of $R^2$ is

\begin{equation}
\label{rho2}
\rho^2 = \frac
{E \left [(V-EV)^2 \right ] - E \left [\{V-E(V|U)\}^2 \right ]
}
{E \left [(V-EV)^2 \right ]}
\end{equation}

\noindent
This is the proportional reduction in mean squared prediction error
attained by using U, versus not using it.\footnote{The quantity $\rho$
can be shown to be the correlation between V and $E(V|U)$.}  

The analog of this quantity in the case of binary V could be defined to
be the proportional reduction in classification error.  To determine
this, note that without U, we always predict V to be 1 if $EV \geq 0.5$;
otherwise our prediction is always 0.  The probability of
misclassification is then $\min(EV,1-EV)$.  Similarly, the conditional
probability of misclassification, given U, is $\min[E(V|U),1-E(V|U)]$,
and the unconditional probability, still using U, is the expected value
of that quantity.

\noindent
Accordingly, we take our binary-V analog of $R^2$ to be

\begin{equation}
\eta = \frac
{\min(EV,1-EV) - E \left [ \min[E(V|U),1-E(V|U) \right ]}
{\min(EV,1-EV)}
\end{equation}

\noindent
Applying Jensen's Inequality to the convex function $\phi(t) =
-\min(t,1-t)$, and using $EV = E[E(V|U)]$, we have that 

\begin{equation}
E \left [ \min[E(V|U),1-E(V|U) \right ] \leq \min(EV,1-EV)
\end{equation}

\noindent
Thus $0 \leq \eta \leq 1$, as with $R^2$.  One could use $\eta$ as a
means of assessing whether to make use of unlabeled data.  


Another approach, motivated by (\ref{improvement}), would be to use

\begin{equation}
\sigma_j = Var[g(X,\theta_j)
\end{equation}

\noindent
Note that in the extreme case $\sigma_j = 0$, $g(\cdot)$ is a constant
and $EV = E(V|U)$, so unlabeled data would be useless.  Note too that
since $0 \leq g \leq 1$, this measure $\sigma_j$ is in some sense
scale-free.  Accordingly, we might base our decision as to whether to 
use unlabeled data on (\ref{varghat}), making use of the unlabeled 
data if this quantity is sufficiently far from 0.

\section{Empirical Evaluation}
\label{empirical}

Rather than merely presenting some examples of our methodology's use on
real data, we aim here to assess the methodology's {\it performance} on
real data.  The bootstrap will be used for this
purpose.  (The reader should note that this is intended just a means of
assessing the general value of the methodology, rather than part of the
methodology itself.)

Suppose we wish to study the efficiency of estimating some population
value $\nu$ using $\widehat{\nu}_s$, an asymptotically unbiased function
of a random sample of size s from the population.  We are interested in
studying the efficiency of $\widehat{\nu}_s$ for various sample sizes s,
as measured by mean squared error

\begin{equation}
\label{msenu}
MSE(\widehat{\nu}_s) = E \left [ (\widehat{\nu}_s -\nu)^2 \right ]
\end{equation}

Now suppose we have an actual sample of size h from the population.  We
take this as the ``population,'' and then draw m subsamples of size s.
In the i$^{th}$ of these subsamples, we calculate $\widehat{\nu}_s$,
denoting its value by $\widehat{\nu}_{si}$.
The point is that the empirical c.d.f. of the $\widehat{\nu}_{si}$, i =
1,...,m, is an approximation to the true c.d.f. of $\widehat{\nu}_s$.
Moreover,

\begin{equation}
\widehat{MSE}(\widehat{\nu}_s) = 
\frac{1}{m} \sum_{i=1}^m (\widehat{\nu}_{si} - \widehat{\nu}_h)^2
\end{equation}

\noindent
is an estimate of the true mean squared estimation error, (\ref{msenu}).

This bootstrap framework was used to study the the efficiency of our
classification probability estimator $\widehat{q}_j$.  On each of three
real data sets, the ratio was computed of the estimated MSE for
$\widehat{q}_j$ to that resulting from merely using $\bar{Y}^{(j)}$ and
the labeled data.

Figure \ref{pima} below shows the estimated MSE ratios for the Pima
Indians diabetes data in the UCI Machine Learning Repository,
\cite{uci}.  (The figures here have been smoothed using the R {\bf
lowess()} function.) Here we are predicting whether a woman will develop
diabetes, using glucose and BMI as features.  The methodology developed
here brings as much as a 22 percent improvement.  

As expected, the improvement is better for larger values of n-r.  In
other words, for each fixed number r of labeled observations, the larger
the number n-r of unlabeled observations, the better we do.  On the
other hand, for each fixed number of unlabeled observations, the smaller
the number of labeled observations, the better (since n-r is a larger
proportion of r).

\begin{figure}[tb]
\centerline{               
\includegraphics[width=5.0in]{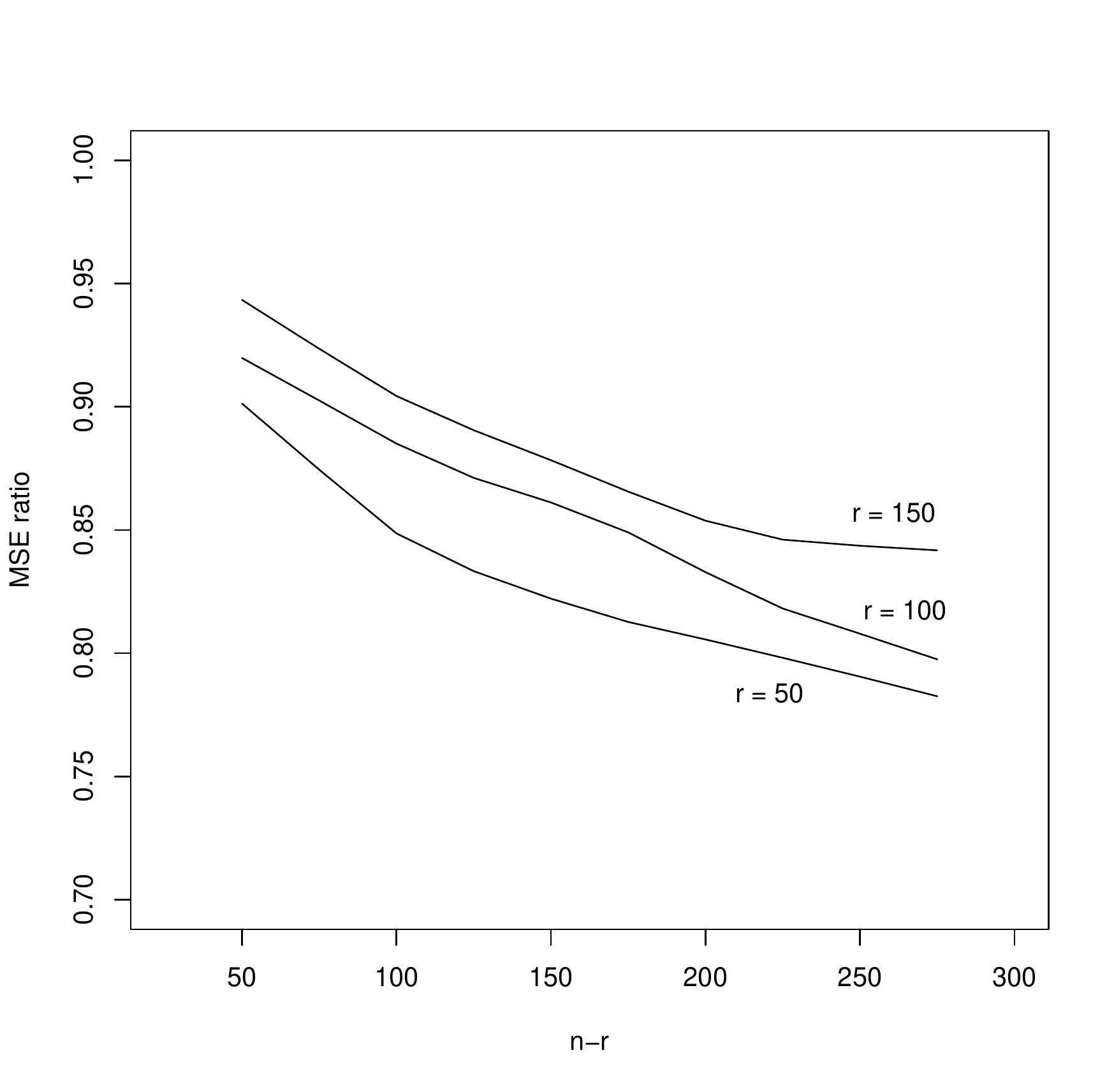}
}
\caption{Pima MSEs}  
\label{pima} 
\end{figure}

Figure \ref{abalone} shows similar behavior for another famous UCI data
set, on abalones.  We are predicting whether the animal has less than or
equal to nine rings, based on length and diameter.

\begin{figure}[tb]
\centerline{               
\includegraphics[width=5.0in]{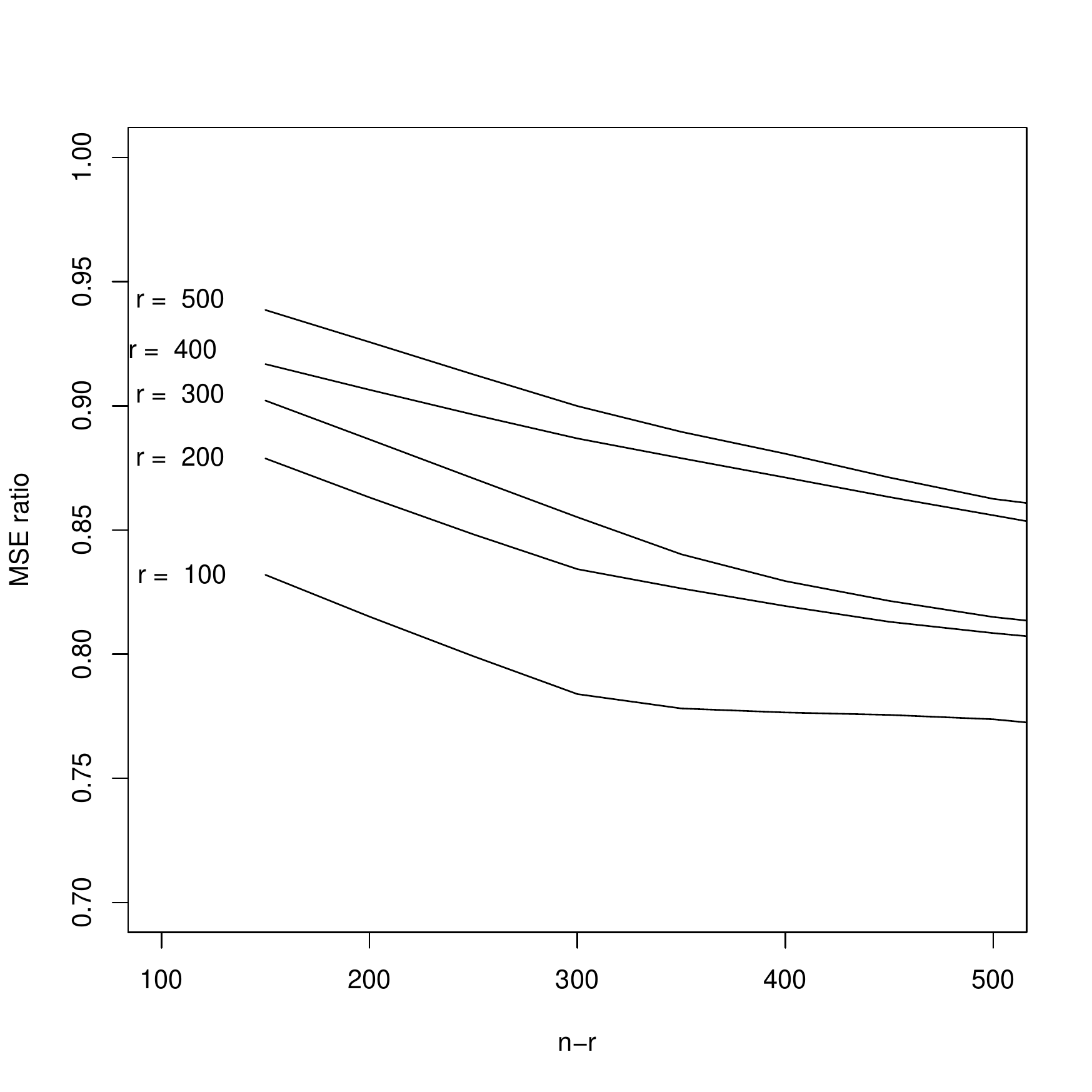}
}
\caption{Abalone MSEs}  
\label{abalone} 
\end{figure}


As pointed out in Section \ref{assessing}, our method here may produce
little or no improvement for some data sets.  Here is an example, from
the 2000 census data, \cite{pums}.  The data consist of records on
engineers and programmers in Silicon Valley, and we are using age and
income to predict whether the worker has a graduate degree.  The
results, seen in Figure \ref{pumsfig}, show that our methodology
produces essentially no improvement in this case.

\begin{figure}[tb]
\centerline{               
\includegraphics[width=5.0in]{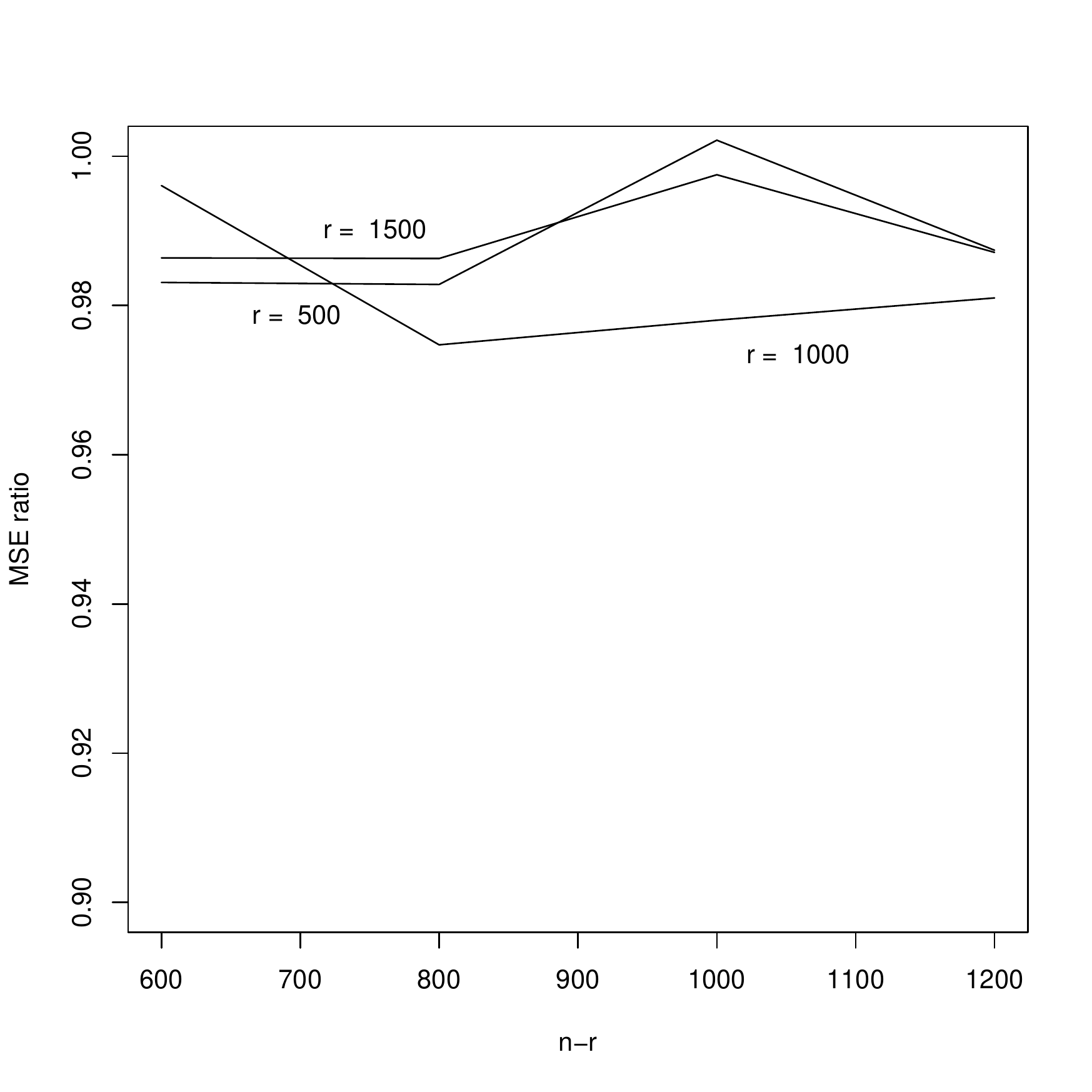}
}
\caption{PUMS MSEs}  
\label{pumsfig} 
\end{figure}

Now let us examine the two possible measures proposed in Section
\ref{assessing}, in the context of the above three data sets.  Table
\ref{worthit} shows the results, which are rather striking.  Even though
the misclassification rate, $\eta$, was approximately equal across the
three data sets, the value of $\sigma$ was an order of magnitude
smaller for PUMS than for the other two data sets.  The latter pattern
exactly reflects the pattern we saw in the MSE values in the figures.

In other words, classification power by itself does not suffice to
indicate whether our method yields an improvement.  It is also necessary
that the feature vector $X$ have a substantial amount of variation.  The
measure $\sigma_j$ provides a direct assessment of the potential for
improvement.

\begin{table}
\begin{center}
\begin{tabular}{|l|l|l|}
\hline
data set & $\hat{\eta}_j$ & $\hat{\sigma}_j$ \\ \hline 
Pima & 0.2403 & 0.0617 \\ \hline 
abalone & 0.2621 & 0.0740 \\ \hline 
PUMS & 0.2177 & 0.0063 \\ \hline 
\end{tabular}
\end{center}
\caption{Comparison of Measures Proposed in Section \ref{assessing}}
\label{worthit}
\end{table}

\section{Estimation of Subclass Probabilities} 
\label{subclass}

Here the focus will be on quantities of the form

\begin{equation}
\label{sbc}
q_{j,W} = P(Y^{(j)} = 1 | X \epsilon W)
\end{equation}

\noindent
for various sets W of interest.  It was mentioned earlier, for instance,
that one may wish to compare the proportions of public funding for
subcategories of patents, say computers and optics.  Here we outline
how the methodology developed earlier in this paper can be extended to
this situation, focusing on the semiparametric case.

Write (\ref{sbc}) as

\begin{equation}
q_{j,W} = \frac{E \left [Y^{(j)} I_W(X) \right ]}{P(X \epsilon W)}
\end{equation}

\noindent
where $I_W$ is the indicator function for membership in W.  By analogy
with (\ref{qhatj}), then, our estimator for $q_{j,W}$ is

\begin{equation}
\label{qhatjw}
\hat{q}_{j,W} = 
\frac
{\sum_{i=1}^n g(X_i,\hat{\theta}_j) I_W(X_i)}
{\sum_{i=1}^nI_W(X_i)}
\end{equation}

\noindent
The ``ordinary'' estimator, based only on labeled $Y$ values, is

\begin{equation}
\label{ordinary}
\bar{Y}^{(j,W)} = 
\frac
{\frac{1}{r}\sum_{i=1}^r I_W(X_i) Y_j^{(j)}}
{\frac{1}{n}\sum_{i=1}^nI_W(X_i)}
\end{equation}

\begin{theorem}

The asymptotic variance of $\hat{q}_{j,W}$ is

\begin{equation}
AVar(\hat{v}_{j,W}) = 
\frac{1}{n P(X \epsilon W)^2} 
Var\left [ g(X,\theta_j) I_W(X) \right ] + 
\frac{1}{r} 
C_j^T A_j^{-1} C_j 
\end{equation}

\end{theorem}

\begin{proof}
\noindent
Since

\begin{equation}
\label{zz}
\lim_{n \rightarrow \infty}
{\frac{1}{n} \sum_{i=1}^nI_W(X_i)} = P(X \epsilon W) \textrm{ w.p. 1}
\end{equation}

\noindent
we can again apply Slutsky's Theorem and derive the asymptotic
distribution of $\hat{q}_{j,W}$.  Toward that end, define 

\begin{equation}
v_{j,W} = E \left [Y^{(j)} I_W(X) \right ]
\end{equation}

\noindent
and concentrate for the moment on estimating $v_{j,W}$, using

\begin{equation}
\hat{v}_{j,W} = 
\frac{1}{n} \sum_{i=1}^n g(X_i,\hat{\theta}_j) I_W(X_i)
\end{equation}

\noindent
The analog of (\ref{taylor}) is then

\begin{eqnarray}
n^{\frac{1}{2}} (\hat{v}_{j,W} - v_{j,W}) &=&
n^{-\frac{1}{2}}
\sum_{i=1}^n
\left [ g(X_i,{\theta_j}) I_W(X_i) - v_{j,W} \right ]  \\
&+&
\left [  
n^{-1} \sum_{i=1}^n g'(X_i,\tilde{\theta}_j) I_W(X_i)
\right ]^T ~
n^{\frac{1}{2}} (\hat{\theta}_j - \theta_j ) \nonumber
\end{eqnarray}

\noindent
for some $\tilde{\theta}_j$ between $\theta_j$ and $\hat{\theta}_j$. 

Proceeding as in (\ref{qasy}) we have that $n^{\frac{1}{2}}  (\hat{v}_{j,W} -
v_{j,W})$ is asympotically equivalent to

\begin{equation}
n^{-\frac{1}{2}}
[
\sum_{i=1}^n
\{ g(X_i,{\theta_j})I_{j,W}(X_i) - v_{j,W} \}  ~ + \\
\gamma^{-\frac{1}{2}}
C_j^T 
A_j^{-1} 
\sum_{i=1}^r 
w(X_i,{\theta_j}) 
\{
Y_i^{(j)} - g(X_i,{\theta_j})
\} ~
g'(X_i,{\theta_j})
] 
\end{equation}

\noindent
where

\begin{equation}
C_j = E \left [ g'(X,\theta_j) I_W(X_i) \right ]
\end{equation}

\noindent
Thus $n^{\frac{1}{2}}  (\hat{q}_{j,W} - q_{j,W})$ is asympotically equivalent 
to

\begin{equation}
n^{-\frac{1}{2}} \frac{1}{P(X \epsilon W)}
\left [
\sum_{i=1}^n
\{ g(X_i,{\theta_j})I_{j,W}(X_i) - v_{j,W} \}  ~ + \\
\gamma^{-\frac{1}{2}}
C_j^T 
A_j^{-1} 
\sum_{i=1}^r 
w(X_i,{\theta_j}) 
\{
Y_i^{(j)} - g(X_i,{\theta_j})
\} ~
g'(X_i,{\theta_j})
\right ] 
\end{equation}

Continuing as before, the desired result follows:

\begin{equation}
AVar(\hat{v}_{j,W}) = 
\frac{1}{P(X \epsilon W)^2} 
\left [
\frac{1}{n} Var\left [ g(X,\theta_j) I_W(X) \right ] + 
\frac{1}{r} 
C_j^T A_j^{-1} C_j 
\right ]
\end{equation}

\end{proof}

\begin{corollary}

The asympotic variance of the estimator based on both the labeled and
unlabeled data, $\widehat{q}_{j,W}$, is less than or equal to that of
the estimator based only on labeled data, $\bar{Y}^{(j,W)}$.  The
inequality as strict as long as the random variable $I_W(X)
g(X,\theta_j)$ is not constant.

\end{corollary}

\begin{proof}
\noindent
Define $h(t,\theta) = g(t,\theta) I_W(t)$.  Theorem \ref{thm1}, applied
to h and $1_W(X) X$ rather than g and $X$, shows that

\begin{equation}
\frac{1}{n} \sum_{i=1}^n g(X_i,\hat{\theta}_j) I_W(X_i)
\end{equation}

\noindent
has asympotic variance less than that of 

\begin{equation}
\frac{1}{r} \sum_{i=1}^r I_W(X_i) Y_j^{(j)}
\end{equation}

\noindent
The result then follows by applying Slutsky's Theorem to the
denominators in (\ref{qhatjw}) and (\ref{ordinary}), and noting
(\ref{zz}).

\end{proof}

\section{The Case of Purely Discrete Features}
\label{discrete}

Suppose $X$ is discrete, taking on b vector values $v_1,...,v_b$.  We
assume here that our labeled data is extensive enough that for each k,
there is some i with $X_i = v_k$.  Then we can find direct estimates of
the $q_j$ without resorting to using a parametric model or smoothing
methods such as nearest-neighbor.

We can write the prior probabilities as

\begin{equation}
\label{prioreqn}
q_j = \sum_{k=1}^b p_k d_{jk}
\end{equation}

\noindent
where $p_k = P(X = v_k)$ and $d_{jk} = P(Y^{(j)} = 1 | $X$ = v_k)$.

Let $1_{ik}$ be the indicator variable for $X_i = v_k$, for i = 1,...,n
and k = 1,...,b.  Then denote the counts of labeled and unlabeled
observations taking the value k by

\begin{equation}
M_{k} = \sum_{i=1}^r 1_{ik} 
\end{equation}

\noindent
and

\begin{equation}
N_{k} = \sum_{i=r+1}^n 1_{ik} 
\end{equation}

\noindent
Also define the count of class-j observations among the labeled data
having feature value k:

\begin{equation}
T_{jk} = \sum_{i=1}^r 1_{ik} Y_i^{(j)}
\end{equation}

\noindent
We can estimate the quantities $d_{jk}$ and $p_k$ by

\begin{equation}
\hat{d}_{jk} = \frac{T_{jk}}{M_k}
\end{equation}

\noindent
and

\noindent
\begin{equation}
\hat{p}_k = \frac{M_k+N_k}{n} 
\end{equation}

\noindent
Our estimates of the class prior probabilities are then, in analogy to 
(\ref{prioreqn}),

\begin{equation}
\label{ourest}
\hat{q}_j 
= \sum_{k=1}^b \hat{p}_k ~ \hat{d}_{jk}
\end{equation}

Note that if we do not use unlabeled data, this estimate reduces to
the standard estimate of $q_j = EY^{(j)}$, the class sample proportion
$\bar{Y}^{(j)}$, in (\ref{ybarj}).

The quantities $M_h$, $N_h$ and $T_{jh}$ have an asympotically
multivariate normal distribution.  Thus from the multivariate version of
Slutsky's Theorem is asympotically equivalent to

\begin{equation}
\label{gjhconst}
\sum_{k=1}^b d_{jk} ~ \hat{p}_k 
\end{equation}

This is a key point in addressing the question as to whether the
unlabeled data provide our estimators with smaller asympotic variance.
Since the only randomness in (\ref{gjhconst}) is in the $\hat{p}_k$, the
larger the sample used to calculate those quantities, the smaller the
variance of (\ref{gjhconst}) will be.  Using the unlabeled data provides
us with that larger sample, and thus (\ref{gjhconst}) will be superior
to (\ref{ybarj}).

Stating this more precisely, the variance of (\ref{gjhconst}) is

\begin{equation}
\label{ourestav}
\frac{1}{n}
\left (
\sum_{k=1}^b d_{jk}^2 ~ p_k(1-p_k) -
2 \sum_{k=1}^b \sum_{s=k+1}^b d_{jk} d_{js} p_k p_s
\right )
\end{equation}

\noindent
Without the unlabeled data, this expression would be the same but with a
factor of 1/r instead of 1/n.  So, again, it is demonstrated that use of
the unlabeled data is advantageous.  

To form confidence intervals and hypothesis tests from (\ref{ourest}),
we obtain a standard error by substituting $\hat{d}_{jh}$ and $\hat{p}_h$
in (\ref{ourestav}), and taking the square root.

\section{Conclusions and Discussion} 
\label{dsc}

This paper has developed methodology for estimating class prior
probabilities in situations in which class membership is
difficult/expensive to determine.  Asymptotic distributions of the
estimators were obtained, enabling users to form confidence intervals
and perform hypothesis tests.  It was proven that use of unlabeled data
does bring an improvement in asymptotic variance, compared to using only
the labeled data.

In the parametric cases, the usual considerations for model fitting
hold.  One should first assess the goodness of fit of the model 
Though many formal test procedures exist, 
with
the large samples often encountered in classification problems it may be
preferable to use informal assessment.  One can, for instance, estimate
the function $E(Y^{(j)} | X = t)$ nonparametrically, using say kernel or
nearest-neighbor methods, and then compare the estimates to those
resulting from a parametric fit.

The paper first built from the semiparametric framework of
\cite{matloff}, and then also derived a fully nonparametric estimator
for the case of purely discrete features.  The latter methodology
requires that the training set be large enough that there is at least
one observation for every possible value of X.  If there are many such
values, or if the training set is not large enough, one might turn to
smoothing methods, such as kernel-based approaches.  Another possibility
in the purely discrete case is to use the log-linear model
[\cite{agresti}], which is fully parametric.  

One might pursue possible further improvement by using the unlabeled data
not just to attain a better estimate for the marginal distribution of
$X$, but also to enhance the accuracy of the estimate of the conditional
distribution of $Y$ given $X$.  As noted in Section \ref{intro}, much
work has been done on this problem in the context of estimating
classification functions.  It would be of interest to investigate the
asymptotic behavior in our present context of estimation of class prior
probabilities.  It would also be useful to investigate whether the
methods of Section \ref{assessing} can be applied to the classification
function problem.



{}

\end{document}